\documentclass[twoside,leqno,twocolumn]{article}
\usepackage{ltexpprt}
\usepackage{amssymb, amsmath, amsfonts}
\usepackage{thmtools}

\usepackage{cite}
\usepackage{dsfont}
\usepackage{footnote}

\usepackage{subcaption}
\usepackage{graphicx}

\usepackage{algorithm}
\usepackage{algorithmic}

\usepackage{textcomp}
\usepackage{multirow}

\usepackage{color}

\graphicspath{{./Fig/}}

\DeclareMathOperator*{\argmin}{argmin}
\begin{document}

\title{\Large Strongly Hierarchical Factorization Machines and ANOVA Kernel Regression\thanks{Supported by DoD Minerva program and AFOSR (grant FA9550-15-1-0159)}}
\author{Ruocheng Guo\thanks{Arizona State University. \{rguo12,halvari,shak\}@asu.edu} \\
\and
Hamidreza Alvari \footnotemark[2] \\
\and
Paulo Shakarian\footnotemark[2] \\
}
\date{}

\maketitle


\fancyfoot[R]{\footnotesize{\textbf{Copyright \textcopyright\ 20XX by SIAM\\
Unauthorized reproduction of this article is prohibited}}}





\begin{abstract}
	%
	High-order parametric models that include terms for feature interactions are applied to various data mining tasks, where ground truth depends on interactions of features.
	However, with sparse data, the high-dimensional parameters for feature interactions often face three issues: expensive computation, difficulty in parameter estimation and lack of structure.
	Previous work has proposed approaches which can partially resolve the three issues.
	%
	In particular, models with factorized parameters (e.g. Factorization Machines) and sparse learning algorithms (e.g. FTRL-Proximal) can tackle the first two issues but fail to address the third. 
	Regarding to unstructured parameters, constraints or complicated regularization terms are applied such that hierarchical structures can be imposed.
	However, these methods make the optimization problem more challenging.
	In this work, we propose Strongly Hierarchical Factorization Machines and ANOVA kernel regression where all the three issues can be addressed without making the optimization problem more difficult.
	Experimental results show the proposed models significantly outperform the state-of-the-art in two data mining tasks: cold-start user response time prediction and stock volatility prediction.

\end{abstract}
\section{Introduction}
\label{sec:intro}
%
%
%
%
%
%

In the area of data mining, there exist many high-order parametric models, which explicitly incorporate terms for modeling interactions between features.
In the applications such as prediction of users' behavior in social media~\cite{rendle2012social,hong2013co,zhu2012discovering}, movie ratings~\cite{rendle2010factorization} and stock return volatility~\cite{blondel2016polynomial}, co-occurrence of features can be crucial to decide the ground truth labels. 
For example, the observation that a user retweeted the microblog \textit{SIAM SDM deadlines are approaching.} just a moment after it had been posted can result from the co-occurrence of the word \textit{SDM} and the phrase \textit{data mining researcher} in the user's profile. 
In the case of models including feature interactions, the problem is to learn a function that maps features and their interactions to a scalar such that a predefined loss function is minimized.
In the most straightforward high-order models (i.e. Polynomial Regression), each feature interaction is modeled by an independent parameter.
This leads to the problem of high-dimensional parameters.
In sparse and high-dimensional settings where the number of nonzero elements in each feature vector is much smaller than its dimension, there are mainly three issues:
	1. Expensive computation: the number of parameters increases exponentially with the order number.
	2. Difficulty in estimating high-dimensional parameters with sparse data: for example, given a pair of features $(x_i,x_j)$, it requires enough samples with $x_ix_j \not = 0$ for reliable parameter estimation which is not the case in sparse data.
	3. Lack of structure between parameters: it is hard to justify models where interaction $x_ix_j$ plays an important role in prediction but neither $x_i$ nor $x_j$ does.
	%
%

To address the first two issues, one idea is to learn sparse models. Approaches such as Lasso~\cite{tibshirani1996regression} and Elastic net~\cite{zou2005regularization} were proposed, which apply regularization terms (e.g. $L_1$ norm of parameters) that can lead to sparsity.
Recently, follow-the-regularized-leader algorithms such as RDA~\cite{xiao2010dual} and FTRL-Proximal~\cite{mcmahan2011follow} have been shown to be effective in producing sparsity for generalized linear models.
%
%
%
Another idea is to develop novel models that can handle interaction effects with low-dimensional parameters.
Models such as Factorization Machines (FMs) and ANOVA kernel regression~\cite{rendle2010factorization,blondel2016polynomial} resolve the first two issues by modeling interaction effects with low-rank factorized parameters.

As shown in~\cite{liu2016efficient,bien2013lasso,choi2010variable,zhong2013efficient}, hierarchical structures between main and interaction effects contribute to models effectiveness and selection of important features and interactions.
To address the third issue, Bien et al.~\cite{bien2013lasso} defined strong and weak hierarchy.
Strong (weak) hierarchy demonstrates that an interaction effect could be with non-zero weight iff both (one) of the corresponding linear terms are.
%
%
They also proposed the Weak Hierarchical Lasso where constraints are added to the optimization problem to guarantee weak hierarchy.
Then, in~\cite{li2016difacto}, Li et al. mentioned a structured sparsity which can impose strong hierarchy to their model with a complicated regularization term.
However, both methods mentioned above can limit two types of operations :
applying other regularization or constraints for specific purpose (e.g. domain prior knowledge),
usage of efficient algorithms which are only applicable for certain types of loss functions (e.g. convex functions) or regularizations (e.g. $L_1$ regularization).
The facts listed above motivate us to propose our models. 
We show that strong hierarchy can be imposed by adding a context dimension to FMs and ANOVA kernel regression.
Thus, we can address the three issues simultaneously and leave the optimization problem without extra constraints or complicated regularization terms.
We list our contributions as below:
\begin{itemize}
	\item We propose Strongly Hierarchical FMs and ANOVA kernel regression where the three issues mentioned above are addressed without making the optimization problem more difficult.
	\item We show that predictions can be made with these models, time complexity is linear to data dimension (or average number of nonzero features for sparse data) and the number of latent dimensions.
	\item We derive a FTRL-Proximal class algorithm for the proposed models.
	Analysis shows that the time and space complexity of this algorithm is linear to the feature dimension and the number of latent factors.
	\item Experimental results show our proposed models can significantly outperform the state-of-the-art ones.
	Moreover, these models can achieve high sparsity without significant loss of performance. 
\end{itemize}
The remaining of the paper is organized as follows: we give
a brief review of background knowledge in Section~\ref{sec:bkg}.
In Section~\ref{sec:oa}, we introduce the proposed models and some properties of them. We also derive and analyze an efficient FTRL-Proximal algorithm for them.
Then, we describe the experimental setup and results in Section~\ref{sec:exp}.
Finally, we summarize related work and conclude this paper in Section~\ref{sec:rw} and~\ref{sec:conclu}, respectively.

\section{Background}
\label{sec:bkg}
In this section, we start with the preliminaries. 
Next, we describe hierarchical structures amongst parameters.
Finally, we introduce FMs~\cite{rendle2010factorization} and ANOVA kernel regression~\cite{blondel2016polynomial}. 
\subsection{Preliminaries}
Bold uppercase (e.g. $\boldsymbol{W}$), bold lowercase (e.g. $\boldsymbol{x}$) and lowercase letters (e.g. $y$) denote matrices, vectors and scalars, respectively.
Subscripts refer to rows, columns of matrices or elements of matrices and vectors.
For example, notations $\boldsymbol{W}_{i,:}$ and $\boldsymbol{W}_{:,j}$ denote the $i$th row and $j$th column of matrix $\boldsymbol{W}$, respectively.
Superscripts denote the order number of kernels (e.g. $\mathcal{A}^2$) or the iteration number (e.g. $\boldsymbol{G}^{1:t}$).
\subsection{Factorization Machines and ANOVA kernel regression}

Previous work ~\cite{juan2016field,liu2016efficient,qiang2013exploiting} has shown that augmenting feature vector $\boldsymbol{x}$ with interaction effects $x_ix_j$ can significantly improve performance of models in various data mining tasks such as user response prediction and microblog retrieval.
One of the simplest models which take into account interaction effects is the polynomial regression (PR). With second-order PR, predictions are made as follows:
\begin{equation}
\hat{y}_{PR}(\boldsymbol{x}) = b+ \langle\boldsymbol{\omega},\boldsymbol{x}\rangle+ \sum_{i=1}^{d}\sum_{j=i+1}^{d}{W}_{i,j}x_ix_{j}
\label{eq:PR}
\end{equation}
where $\boldsymbol{\omega} \in \mathds{R}^d$, $\boldsymbol{W} \in \mathds{R}^{d^2}$ and $d$ is the dimension of data.
The number of parameters can be $\mathcal{O}(d^m)$ for $m$th-order PR or SVM with polynomial kernel. 
They are not able to scale well with high-dimensional data. Moreover, this can result in models where significant amount of parameters are fitted only by very few samples because of data sparsity.
Following~\cite{bien2013lasso}, in this paper, we refer to linear terms (e.g. $\langle\boldsymbol{\omega},\boldsymbol{x}\rangle$) as main effects and second-order terms (e.g. $\sum_{i=1}^{d}\sum_{j=i+1}^{d}{W}_{i,j}x_ix_{j}$) as interaction effects.
In~\cite{rendle2010factorization}, Rendle proposed FM by factorizing the parameters for feature interactions $\boldsymbol{W} = \boldsymbol{V}\boldsymbol{V}^T$, where $\boldsymbol{V} \in \mathds{R}^{d\times k}$ and positive integer $k\ll d$ is the number of latent factors. FM is defined as:
\begin{equation}
\hat{y}_{FM}(\boldsymbol{x}) = b +  \langle\boldsymbol{\omega},\boldsymbol{x}\rangle+\sum_{i=1}^{d}\sum_{j=i+1}^d\langle\boldsymbol{v}_i,\boldsymbol{v}_j\rangle{x}_i{x}_j
\label{eq:FM}
\end{equation}
Following~\cite{blondel2016polynomial}, ANOVA kernel regression is defined as:
\begin{equation}
\small
\hat{y}_{\mathcal{A}^2}(\boldsymbol{x}) = b +  \langle\boldsymbol{\omega},\boldsymbol{x}\rangle+\sum_{f=1}^{k}\boldsymbol{\beta}_f\mathcal{A}^2(\boldsymbol{V}_{:,f},\boldsymbol{x})
\label{eq:A2}
\end{equation}
where $\mathcal{A}^m(\boldsymbol{\cdot},\boldsymbol{\cdot})$ is the $m$th-order ANOVA kernel proposed in~\cite{stitson1999support}.
Given vectors $\boldsymbol{a},\boldsymbol{b} \in \mathds{R}^d$ and $d > m$, the ANOVA kernel is formally defined as:
\begin{equation}
\small
\mathcal{A}^m(\boldsymbol{a},\boldsymbol{b}) = \sum_{i_1=1}^d...\sum_{i_m>i_{m-1}}^da_{i_1}b_{i_1}...a_{i_m}b_{i_m}
\end{equation}
We denote second-order ANOVA kernel regression model by symbol $\mathcal{A}^2$.
In~\cite{blondel2016polynomial}, it is shown that FM is a special case of ANOVA kernel regression with $\boldsymbol{\beta} = \boldsymbol{1}$.

\subsection{Hierarchical Structures for Parameters}
	%
	
	Bien et al. defined strong and weak hierarchy in~\cite{bien2013lasso}.
	Using notations from (\ref{eq:PR}), they are defined as:
	\begin{Definition}
	Strong hierarchy: ${W}_{i,j} \not = 0\Rightarrow {\omega}_i \not = 0 \; and \; {\omega}_j \not = 0$
	Weak hierarchy: ${W}_{i,j} \not = 0 \Rightarrow {\omega}_i \not = 0 \; or \; {\omega}_j \not = 0$
	\end{Definition}
	Here, we demonstrate the intuition of these two constraints by the example given in the introduction.
	In prediction of when a given user would retweet a particular tweet, it is difficult to justify a model which states that the co-occurrence of the phrase \textit{data mining researcher} in user profile and the abbreviation \textit{SDM} in tweet text is crucial but ignores main effect of either of them.
	%
	%
	%
	%
	%
	%
	%
	In~\cite{bien2013lasso}, it is shown that weak hierarchy can be imposed by adding constraints $||\boldsymbol{W}_{:,j}||_1 <
		|{\omega}_j| \; for \; j=1,...,d$ to the optimization problem with $L_1$ regularization.
	where $\mathcal{L}(y,\hat{y})$ is the loss function and $||\boldsymbol{W}_{:,j}||_1=\sum_{i=1}^d|W_{i,j}|$.
	On the other hand, in~\cite{li2016difacto}, Li et al. mentioned the method to impose strong hierarchy to FMs with the structured sparsity regularization term~\cite{negahban2009unified}:
	$\sum_{i=1}^d\{[\boldsymbol{\omega}_i^2+||\boldsymbol{W}_{:,i}||_2^2]^{\frac{1}{2}}+||\boldsymbol{W}_{:,i}||_2\}.$
	As shown above, the optimization problem would become more challenging if such methods are applied. This is the main motivation for us to propose our models.
\section{Strong Hierarchy with Context Dimension}
\label{sec:oa}

	In this section, we begin with derivation of our proposed model where strong hierarchy is imposed by incorporating the \textit{context dimension} into FM and ANOVA kernel regression.
	Then we demonstrate efficient computation can be done with these models.
	Finally, we derive and analyze the per-coordinate FTRL-Proximal algorithm for our proposed models.
	Since models such as FM can be extended to high-order, without loss of generality, we focus on second-order models in this paper.	
	

	\subsection{The Proposed Models}
	Here, we propose \textit{Strongly Hierarchical ANOVA kernel regression} (SH$\mathcal{A}^2$) and its special case \textit{Strongly Hierarchical Factorization Machines} (SHFMs).
	First, we show how strong hierarchy of these two models are guaranteed without extra constraints or regularization terms. 
	%
	%
	In the proposed models, to create hierarchical structure, parameters for main and interaction effects are connected by considering main effects as interactions between features and the constant \textit{context feature} ($x_0=1$). Then, main effects become $\sum_{i=1}^{d}\langle\boldsymbol{v}_i\odot\boldsymbol{\beta},\boldsymbol{v}_0\rangle x_ix_0$, where $\boldsymbol{v}_0$ is the \textit{context latent factor}.
	Moreover, we can merge the main effects into interactions by concatenating the \textit{context dimension} ($x_0$, $\boldsymbol{v}_0$) to features and parameters, respectively: $\boldsymbol{V}^{\prime} = [\boldsymbol{v}_0,\boldsymbol{v}_1,...,\boldsymbol{v}_d]$ and $\boldsymbol{x}^{\prime}=[x_0,x_1,...,x_d]$.
	%
	%
	Finally, we formulate SH$\mathcal{A}^2$ as:
	\begin{equation}\label{eq:A2SH}
	\small
	\begin{split}
	\hat{y}_{SH\mathcal{A}^2}(\boldsymbol{x}) &= b+ \sum_{i=0}^{d}\sum_{j=i+1}^{d}\langle\boldsymbol{v}_i\odot\boldsymbol{\beta},\boldsymbol{v}_j\rangle x_ix_j \\ &=  b +  \sum_{f=1}^k\beta_f\mathcal{A}^2(\boldsymbol{V}_{:,f}^{\prime},\boldsymbol{x}^{\prime})\\
	\end{split}
	\end{equation}
	In~\cite{blondel2016polynomial}, authors proved that fitting $\boldsymbol{\beta}$ can be helpful when the order number is even.
	Therefore, we consider both the cases: the parameter vector $\boldsymbol{\beta}$ takes constant value $\boldsymbol{1}$ (SHFMs); $\boldsymbol{\beta}$ is estimated as parameters (SH$\mathcal{A}^2$).
	%
	%
	\begin{proposition}
		Strong hierarchy is guaranteed in parameters of SHFMs with assumptions: 1). $\boldsymbol{v}_0 \not = \boldsymbol{0}$; 2). $\boldsymbol{v}_0 \not \perp \boldsymbol{v}_i \; , i = 1,...,d$.
		\label{prop:sh}
	\end{proposition}
	\begin{proof}
		For a pair of features $(x_i,x_j)$, given $\langle\boldsymbol{v}_i,\boldsymbol{v}_j\rangle \not = 0$, we can infer that $\boldsymbol{v}_i \not = \boldsymbol{0}$ and $\boldsymbol{v}_j \not = \boldsymbol{0}$.
		Further,
		we can conclude $\langle\boldsymbol{v}_i,\boldsymbol{v}_0\rangle \not = 0$ ($\langle\boldsymbol{v}_j,\boldsymbol{v}_0\rangle \not = 0$) by $\boldsymbol{v}_i \not = \boldsymbol{0}$ ($\boldsymbol{v}_j \not = \boldsymbol{0}$) and the two assumptions.
		Therefore, strong hierarchy is guaranteed by
		$\langle\boldsymbol{v}_i,\boldsymbol{v}_j\rangle \not = 0 \Rightarrow \langle\boldsymbol{v}_i,\boldsymbol{v}_0\rangle \not = 0 \; and \; \langle\boldsymbol{v}_j,\boldsymbol{v}_0\rangle \not = 0$.
		This means that the interaction effect between $x_i$ and ${x}_j$ will be included in the model iff both of their main effects are.
	\end{proof}
	%
	%
	We justify the second assumption of Proposition~\ref{prop:sh} by experiments in Section~\ref{sec:exp} showing the probability of cases where $\boldsymbol{v}_0 \perp \boldsymbol{v}_i$ and $\boldsymbol{v}_0 \not = \boldsymbol{0}$ is significantly lower than those with $\boldsymbol{v}_i = 0$.
	Next, we demonstrate Proposition~\ref{prop:kd} about the time complexity of making a prediction with SH$\mathcal{A}^2$ (\ref{eq:A2SH}).
	This follows the conclusion from~\cite{rendle2010factorization,blondel2016polynomial} that the time complexity of making a prediction with FMs and $\mathcal{A}^2$ can be reduced from $\mathcal{O}(kd^2)$ to $\mathcal{O}(kd)$.
	\begin{proposition}\label{prop:kd}
		Time complexity of making a prediction with SH$\mathcal{A}^2$ is $\mathcal{O}(kd)$. 
	\end{proposition}
	The proof of Proposition~\ref{prop:kd} can be found in Appendix.
		%
%
	Experimental results showing the linear time complexity with the proposed models can be found in Section~\ref{sec:exp}. 
	It worth noting that if $x_i=0$, no computation is needed for the dimension $i$ as all terms w.r.t $x_i$ would be $0$.
	Therefore, when the data is sparse, we can write the right-hand side of (\ref{eq:A2SH}) as: 
	\begin{equation}
	\small
	\label{eq:sparseA2SH}
	\frac{1}{2}\sum_{f=1}^k\beta_f[(\sum_{i \in I}V_{i,f}x_i)^2-\sum_{i\in I}(V_{i,f}x_i)^2] + b
	\end{equation}
	where $I = \left\{0\right\}\cup\left\{x_i \not= 0,i=1,..,d\right\}$. In this way, the time complexity can be reduced to $\mathcal{O}(k\cdot card(I))$, where $card(\cdot)$ denotes the cardinality of a set.
	
	\subsection{Learning SHFMs and SH$\mathcal{A}^2$}
	Because SHFMs is a special case of SH$\mathcal{A}^2$, we focus on learning SH$\mathcal{A}^2$. 
	We only discuss how to learn $\boldsymbol{V}^{\prime}$ as estimating the bias $b$ and the weight vector for each latent dimension $\boldsymbol{\beta}$ is trivial compared to $\boldsymbol{V}^{\prime}$.
	%
	For the proposed models, given any loss function $\mathcal{L}(y,\hat{y}_{SH\mathcal{A}^2})$ convex in predicted label $\hat{y}$, we show that it is also convex along each element of $\boldsymbol{V}^{\prime}$.
	This is done through the demonstration that the model equation of our proposed models is affine to each row of the factorized parameter matrix (Proposition~\ref{clm:affine}).
	\begin{theorem}\label{theo:multiconvex}
		The loss function $\mathcal{L}(y_i,\hat{y}_{SH\mathcal{A}^2}(\boldsymbol{x}))$ is convex in each element of the factorized parameter matrix $\boldsymbol{V}^{\prime}$, assuming $\mathcal{L}(y,\hat{y})$ is convex in $\hat{y}$.
	\end{theorem}
	We start with demonstration of Proposition~\ref{clm:affine}, which enables us to prove Theorem~\ref{theo:multiconvex} later.
	%
	\begin{proposition}
		$\hat{y}_{SH\mathcal{A}^2}$ is an affine function of $\boldsymbol{v}_i, i=0,1,...,d$.
		\label{clm:affine}
	\end{proposition}

	\begin{proof}
		Considering $\boldsymbol{V}^{\prime}$ as variable, 
		we analyze interaction effects  $\sum_{f=1}^k\beta_{f}\mathcal{A}^2(\boldsymbol{V}_{:,f}^{\prime},\boldsymbol{x}^{\prime})$.
		In~\cite{blondel2016polynomial}, Blondel et al. concluded that multi-linearity is a key property of ANOVA kernel, which can be written as:
		\begin{equation*}
		\small
		\mathcal{A}^m(\boldsymbol{V}_{:,f}^{\prime},\boldsymbol{x}^{\prime}) = V_{i,f}x_i\mathcal{A}^m(\boldsymbol{V}_{\neg i,f}^{\prime},\boldsymbol{x}_{\neg i}^{\prime}) + \mathcal{A}^{m-1}(\boldsymbol{V}_{\neg i,f}^{\prime},\boldsymbol{x}_{\neg i}^{\prime})
		\end{equation*}
		where $i\in\left\{0,1,...,d\right\}$ and $\neg i=\left\{0,1,...,d\right\} \setminus i$.
		With this property, by only considering $\boldsymbol{v}_{i}$ ($i$th row of the factorized parameter matrix for interaction effects) as a variable while other rows as constants we analyze the second term on right-hand side of (\ref{eq:A2SH}):
		\begin{equation*}
		\small
		\begin{split}
		& \sum_{f=1}^k\beta_{f}\mathcal{A}^2(\boldsymbol{V}_{:,f}^{\prime},\boldsymbol{x}^{\prime}) \\ & = \sum_{f=1}^k\beta_f[ V_{i,f}x_i\mathcal{A}^2(\boldsymbol{V}_{\neg i,f}^{\prime},\boldsymbol{x}_{\neg i}^{\prime}) \; +  \mathcal{A}^{1}(\boldsymbol{V}_{\neg i,f}^{\prime},\boldsymbol{x}_{\neg i}^{\prime})] \\
		& = \sum_{f=1}^{k}\beta_f(a_fV_{i,f}x_i+b_f)  = \langle\boldsymbol{v}_i,x_i\boldsymbol{a}\odot \boldsymbol{\beta}\rangle + \langle \boldsymbol{\beta},\boldsymbol{b} \rangle
		\end{split}
		\end{equation*}
		where $a_f=\mathcal{A}^2(\boldsymbol{V}_{\neg i,f}^{\prime},\boldsymbol{x}_{\neg i}^{\prime})$ and $b_f= \mathcal{A}^{1}(\boldsymbol{V}_{\neg i,f}^{\prime},\boldsymbol{x}_{\neg i}^{\prime})$ are constants.
		%
		%
		Therefore, $x_i\boldsymbol{a} \odot\boldsymbol{b}$ and $\langle \boldsymbol{\beta},\boldsymbol{b} \rangle$ are both constants, and $\hat{y}_{SH\mathcal{A}^2}$ is an affine function for $\boldsymbol{v}_i, i=0,1,...,d$. This completes the proof of Proposition~\ref{clm:affine}.
	\end{proof}
	\begin{proof}
	As assumed, the loss function $\mathcal{L}(y,\hat{y})$ is convex in $\hat{y}$ (e.g. mean squared error, sigmoid cross entropy etc.).
	Then, according to Proposition~\ref{clm:affine}, we know that $\hat{y}_{SH\mathcal{A}^2}$ is affine in each row of the factorized parameter matrix, and thus also affine in each element ($V_{i,f}$).
	Hence, the loss function $\mathcal{L}(y_i,\hat{y}_{SH\mathcal{A}^2}(\boldsymbol{x};b,\boldsymbol{V}^{\prime}))$ is a composite of convex and affine functions in every $V_{i,f}$, which implies that $\mathcal{L}(y_i,\hat{y}_{SH\mathcal{A}^2}(\boldsymbol{x};b,\boldsymbol{V}^{\prime}))$ is convex in each element of the factorized parameter matrix $\boldsymbol{V}^{\prime}$.
	This completes the proof. 
	\end{proof}
	With Theorem~\ref{theo:multiconvex}, we conclude that the loss function can be optimized efficiently with per-coordinate algorithms. 
	%
	Here, we derive Algorithm~\ref{algo:1} to estimate $\boldsymbol{V}^{\prime}$ for SH$\mathcal{A}^2$ based on the FTRL-Proximal algorithm~\cite{mcmahan2011follow,mcmahan2013ad}.
	%
	
	\noindent\textit{FTRL-Proximal for SH$\mathcal{A}^2$.}
	We first derive the Per-coordinate FTRL-Proximal Algorithm with $L_1$ and $L_2$ Regularization for SH$\mathcal{A}^2$.
	As the $t$th sample received by the model, the algorithm plays the following implicit update:
	\begin{equation}
	\small
	\begin{split}
	\boldsymbol{V}^{\prime t+1} = \underset{\boldsymbol{V}^{\prime}}\argmin(\frac{1}{2}\sum_{s=1}^{t}\sigma^s||\boldsymbol{V}^{\prime}-\boldsymbol{V}^{\prime  s}||_2^2+\\ \boldsymbol{G}^{1:t}\odot\boldsymbol{V}^{\prime}+\lambda_1||\boldsymbol{V}^{\prime}||_1+\frac{\lambda_2}{2}||\boldsymbol{V}^{\prime}||_2^2) \\
	\end{split}
	\label{eq:ftrlstep}
	\end{equation}
	where $\boldsymbol{G}^{1:t}=\sum_{s=1}^t\frac{\partial\mathcal{L}(y_i,\hat{y}_{\mathcal{A}^2SH}(\boldsymbol{x}^{\prime s};b,\boldsymbol{V}^{\prime}))}{\partial\boldsymbol{V}^{\prime}}\bigr\vert_{\boldsymbol{V}^{\prime}=\boldsymbol{V}^{\prime t}}$, $\sum_{s=1}^t\sigma^s = \frac{1}{\eta^t}$ and $\eta^t$ is the non-increasing learning rate, we introduce how to compute them later in this section.
	Here, we state and prove Theorem~\ref{theorem:solution} as below. Theorem~\ref{theorem:solution} is an important property of the FTRL-Proximal algorithm to show its efficiency. 
	\begin{theorem}
		\label{theorem:solution}
		In the unconstrained minimization problem of (\ref{eq:ftrlstep}), each factorized parameter $V_{i,f}^{t+1}$ has a closed form solution. 
	\end{theorem}
	\begin{proof}
	To solve Eq.~\ref{eq:ftrlstep} for each $V_{i,f}^{t+1}$ in closed form, we reformulate it as:
	\begin{equation}
	\small
	\begin{split}
	V_{i,f}^{t+1} = \underset{V_{i,f}}\argmin(\frac{1}{\eta_{i,f}^t}+\lambda_2)[\frac{1}{2}(V_{i,f}+(\frac{1}{\eta_{i,f}^t}+\lambda_2)^{-1}z_{i,f}^t)^2 \\ + (\frac{1}{\eta_{i,f}^t}+\lambda_2)^{-1}\lambda_1|V_{i,f}|] + const \\
	\end{split}
	\label{eq:ftrlstep1}
	\end{equation}
	where $z_{i,f}^t = \sum_{s=1}^t(g_{i,f}^s - \sigma^sv_{i,f}^s)$, $\eta_{i,f}^t$ is a hyper-parameter, namely the per-coordinate learning rate. 
	So right-hand side of Eq.~\ref{eq:ftrlstep1} matches the form of the soft-thresholding operator~\cite{donoho1995adapting}:
	\begin{equation*}
	\small
	\omega^* = \underset{\omega}\argmin\frac{1}{2}(x-\omega)^2+\lambda|\omega| =\begin{cases}
	0 & |x| \le \lambda \\
	x(1-\frac{\lambda}{|x|}) & otherwise \\
	\end{cases}
	\end{equation*}
	With $x = -(\frac{1}{\eta_{i,f}^t}+\lambda_2)^{-1}z_{i,f}^t$, $\lambda = (\frac{1}{\eta_{i,f}^t}+\lambda_2 )^{-1}\lambda_1$ and the fact that $\eta_{i,f}^t >0$, $\lambda_1\ge0$ and $\lambda_2 \ge0$ we have: 
	\begin{equation}
	\small
	V_{i,f}^{t+1} = \begin{cases}
	0 & |z_{i,f}^t| \le \lambda_1 \\
	\frac{(\lambda_1sgn(z_{i,f}^t)-z_{i,f}^t)}{(\frac{1}{\eta_{i,f}^t}+\lambda_2)} & otherwise
	\end{cases}
	\label{eq:updateV}
	\end{equation}
	where $sgn(x) = 1$ for $x\ge0$, $sgn(x)=-1$ for $x<0$.
	With (\ref{eq:updateV}), the proof is completed.
	\end{proof}
	Using the chain rule, the gradient is computed as:
	\begin{equation}
	\small
	\begin{split}
	&g_{i,f}^t = \frac{\partial\mathcal{L}(y^t,\hat{y}_{\mathcal{A}^2SH}(\boldsymbol{x};b,\boldsymbol{V}^{\prime }))}{\partial V_{i,f}}\Bigr|_{\boldsymbol{x}=\boldsymbol{x}^t, \boldsymbol{V}^{\prime}=\boldsymbol{V}^{\prime t}}\\ &=\frac{\partial\mathcal{L}(y^t,\hat{y}_{\mathcal{A}^2SH})}{\partial\hat{y}_{\mathcal{A}^2SH}}\frac{\partial\hat{y}_{\mathcal{A}^2SH}}{\partial V_{i,f}}\Bigr|_{\boldsymbol{x}=\boldsymbol{x}^t, \boldsymbol{V}^{\prime}=\boldsymbol{V}^{\prime t}} \\ &= \frac{\partial\mathcal{L}(y^t,\hat{y}_{\mathcal{A}^2SH})}{\partial\hat{y}_{\mathcal{A}^2SH}}\Bigr|_{\boldsymbol{x}=\boldsymbol{x}^t, \boldsymbol{V}^{\prime}=\boldsymbol{V}^{\prime t}}(\beta_fx_i^t\langle \boldsymbol{V}_{\neg i,f}^{\prime t},\boldsymbol{x}_{\neg i}^t\rangle)
	\end{split}
	\label{eq:grad}
	\end{equation}
	As shown in (\ref{eq:updateV}), $\eta_{i,f}^t$ plays the role of per-coordinate learning rate which controls the magnitude of $V_{i,f}^{t+1}$.
	Following~\cite{mcmahan2013ad}, with positive hyper-parameters $\alpha$, $\mu$ and $\gamma$, we set it as:
	$\eta_{i,f}^t = \frac{\alpha}{(\mu+\sum_{s=1}^t(g_{i,f}^t)^2)^{\gamma}}
	\label{eq:lr}$.
	Therefore, the same initial learning rate $\eta_{i,f}^0=\frac{\alpha}{\mu^\gamma}$ is used for each co-ordinate.
	According to (\ref{eq:grad}), $g^t_{i,f}=0$ when $x_i^t=0$, which means that if the $i$th feature does not occur, then the gradient of each element in $\boldsymbol{v}_i$ is zero.
	Then, with the sum of squared gradients in denominator of $\eta_{i,f}^t$, the more $i$th feature is found in training samples $\boldsymbol{x}^s (s\le t)$, the smaller $\eta_{i,f}^t$ is likely to be.
	According to~\cite{li2016difacto}, this imposes \textit{Frequency Adaptive Regularization} (FAR) to the learning process. FAR refers to applying larger learning rate on parameters corresponding to infrequent features.
	The FAR method has been shown to be effective to improve model's generalized error.
	%
	%
	%
	To explain FAR by our running example, we can state that if an AI has already been familiar with the retweet time pattern of users who describe themselves as \textit{data mining researcher} or tweets related to \textit{SDM}, it does not demand drastic changes for further observation with these two features.
	%
	\begin{algorithm}[t!]
		\small
		\caption{FTRL-Proximal for SH$\mathcal{A}^2$}
		\label{algo:1}
		\begin{algorithmic}[1]
			\renewcommand{\algorithmicrequire}{\textbf{Input:}}
			\renewcommand{\algorithmicensure}{\textbf{Output:}}
			\REQUIRE $\alpha$, $\beta$, $\lambda_1$, $\lambda_2$
			\ENSURE  $\boldsymbol{V}^{\prime}$
			\\ \textit{Init} : $\boldsymbol{Z}=\boldsymbol{0} \in \mathds{R}^{(d+1)\times k}$, $\boldsymbol{N}=\boldsymbol{0} \in \mathds{R}^{(d+1)\times k}$
			\FOR {$t = 1$ to $T$}
			\STATE Receive the sample $\boldsymbol{x}^t$, concatenate it as $\boldsymbol{x}^{\prime t} = [x_0,\boldsymbol{x}^t]$, let $I=\left\{0\right\}\cup \left\{i|x_i^t\not = 0,\; i=1,...,d \right\}$
			\FOR {$f \in \left\{1,2,...,k\right\}$}
			\FOR {$i \in I$}
			\STATE Compute $V_{i,f}^t$ by (\ref{eq:updateV})
			\ENDFOR
			\ENDFOR
			\STATE Compute prediction $\hat{y}_{\mathcal{A}^2SH}^t$ by (\ref{eq:A2SH}) with $\boldsymbol{V}^{\prime t}$
			\STATE Observe label $y^t$
			\FOR {$f \in \left\{1,2,...,k\right\}$}
			\STATE Compute $\langle \boldsymbol{V}^{\prime t}_{:,f},\boldsymbol{x}^{\prime t}\rangle$
			\FOR {$i \in I$}
			\STATE Compute $g_{i,f}^t$ by (\ref{eq:grad}) 
			\STATE $\sigma_{i,f}^t = \frac{1}{\alpha}(\sqrt{n_{i,f}+(g_{i,f}^t)^2}-\sqrt{n_{i,f}})$
			\STATE $z_{i,f} \leftarrow z_{i,f} + g_{i,f}^t - \sigma_{i,f}^tV_{i,f}^t$
			\STATE $n_{i,f} \leftarrow n_{i,f}+(g_{i,f}^t)^2$
			\ENDFOR
			\ENDFOR
			\ENDFOR
		\end{algorithmic} 
	\end{algorithm}
	
	\noindent\textit{Analysis of Algorithm}
	Here we demonstrate the complexity of Algorithm~\ref{algo:1} by the following Proposition.
	\begin{proposition}
		\label{prop:kd2}
		The time and space complexity of Algorithm~\ref{algo:1} is $\mathcal{O}(k\cdot card(I))$ and $\mathcal{O}(dk)$, respectively.
	\end{proposition}
	\begin{proof}
		First, we discuss the time complexity. For each iteration $t$, first, elements of $\boldsymbol{V}^{\prime t}$ are computed in line 5 by Eq.~\ref{eq:updateV} in $\mathcal{O}(k\cdot card(I))$. 
		As given $z_{i,f}$, $\eta_{i,f}^t$ and $\lambda_1$, it takes $\mathcal{O}(1)$ to compute each $V_{i,f}^t$.
		Then, (\ref{eq:A2SH}) in line 8 also takes $\mathcal{O}(k\cdot card(I))$ as shown in (\ref{eq:sparseA2SH}). 
		The algorithm computes $\langle \boldsymbol{V}^{\prime t}_{:,f},\boldsymbol{x}^{\prime t}\rangle$ for each $f$ in line 11, which overall takes $\mathcal{O}(k\cdot card(I))$. By doing this, we avoid repeating the computation of $\langle \boldsymbol{V}^{\prime t}_{:,f},\boldsymbol{x}^{\prime t}\rangle$ in inner loop for each $i \in I$.
		Finally, similar to line 5, 
		each line from line 13 to 16 takes $\mathcal{O}(k\cdot card(I))$ for all $f=1,...,k$ and $i\in I$.
		With these analysis, we conclude that the time complexity of Algorithm~\ref{algo:1} is $\mathcal{O}(7 Tk\cdot card(I))=\mathcal{O}(Tk\cdot card(I))$ with $T$ iterations.
		On the other hand, in terms of space complexity, the two matrices $\boldsymbol{Z}$ and $\boldsymbol{N}$ need to be stored in memory. Besides them, in each iteration, implicitly, it also requires to store $\boldsymbol{V}^{\prime t}$. So, the space complexity of Algorithm~\ref{algo:1} is $\mathcal{O}(3(d+1)k)=\mathcal{O}(dk)$.
	\end{proof}
	It is worthwhile to mention that the space complexity of Algorithm~\ref{algo:1} can not be reduced to $\mathcal{O}(k\cdot card(I))$ because elements in both $\boldsymbol{Z}$ and $\boldsymbol{N}$ accumulate impacts from $x_i^t ,\forall i \in I$ with $t=1,...,T$ where set $I$ can be different for each iteration. 	
	
\section{Experiments}
\label{sec:exp}
	In this section, we start with the dataset description for evaluating the proposed models. 
	Then, we describe our experimental setup.
	Finally, we report experimental results in two aspects: effectiveness, sensitivity analysis.
	
	\subsection{Datasets}
	\label{subsec:data}
	In Table~\ref{tab:data}, we summarize statistics of datasets with three quantities: the dimensionality of data, the number of training and testing samples.
	\begin{table}[tbh!]%
		\small
		\renewcommand{\arraystretch}{1}
		\caption{\textmd{Statistics of datasets}}
		\label{tab:data}
		\centering
		\begin{tabular}{| c| c|c|c|c|}
			\hline
			Dataset &Dimension & Training & Testing   \\ \hline \hline
			WDYR (cold-start) & 24,025 & 78,738 &9,186   \\ \hline
			E2006-tfidf & 150,360 & 16,087 & 3,308  \\ \hline 
		\end{tabular}
	\end{table}

	\noindent\textit{WDYR.} We collect the When Do You Retweet (WDYR) dataset and share a subset of it online\footnote{goo.gl/7Anpkf}.
	The WDYR dataset is a collection of retweets posted from June to November in 2016 and related user profiles.
	The task is to predict how much time it takes for a certain user to retweet a particular original tweet.
	In this dataset, each sample represents a retweet labeled by $t_{r} - t_{0}$ where $t_r$ and $t_0$ are when the retweet and the original tweet was posted, respectively. For each user, we only consider her earliest retweet, given a tweet. We categorize $t_{r} - t_{0}$ (in seconds) into five classes: $t_{r} - t_{0}\in (0,10^3)$, $t_{r} - t_{0}\in[10^3,10^4)$, $t_{r} - t_{0}\in[10^5,10^6)$ and $t_{r} - t_{0}\in(10^6,+\infty)$.
	A retweet can be considered as a result from interactions of the user and the original tweet.
	Thus, we concatenate the user profile features and those of the original tweet.
	User profile includes: user id, user description, create time, favorite count, followers count, friends count and tweet count. Tweet attributes are: original tweet id, original tweet time ($t_0$) and tweet text.
	We treat each attribute as a field and apply one-hot encoding to each field (See Appendix).
	The training set consists of retweets from the first $90\%$ of original tweets.
	Thus, the task can be interpreted as predicting $t_r-t_0$ for unseen original tweets whose retweets can only be in testing set.
	To the best of our knowledge, this is the first study of cold-start problem w.r.t. information diffusion.

	
	\noindent\textit{E2006-tfidf.} E2006-tfidf~\cite{kogan2009predicting} is a subset of the 10K corpus\footnote{http://www.cs.cmu.edu/\texttildelow ark/10K/} of reports from thousands of publicly traded companies in the United States. 
	The target is to predict logarithm scale of stock return volatility (log-volatility) which is often used in the industry of finance to measure risk. As log-volatility takes continuous values, this is a regression task.
	Features comprise tf-idf of unigrams and volatility in the past 12 months.
	
	\subsection{Experimental Setup}
	In the experiments, for each iteration, models are trained with mini-batches, then we evaluate the performance on the complete testing set.
	Generalized linear models are trained with the FTRL-Proximal algorithm for linear models~\cite{mcmahan2011follow,mcmahan2013ad}.
	Our proposed models are trained with Algorithm~\ref{algo:1}. At the same time, FMs and $\mathcal{A}^2$ are trained with a variant of Algorithm~\ref{algo:1} without the \textit{contextual dimension}.
	For hyper-parameters, grid search is carried out. The domain of grid search for each hyper-parameter is shown in Appendix.
%

	%
	\noindent\textit{Loss functions.}
	In the experiments for multi-class classification, the loss function we use is the softmax cross-entropy:
	$\mathcal{L}(y,\hat{y}) = -\sum_{i=1}^c z_i log_{2}(softmax(\hat{y}_i))$
	where $\hat{y}_i$ is the output of a model for the $i$th class, vector $\boldsymbol{z} \in \left\{0,1\right\}^c$ is the one-hot encoding of label $y$ and $softmax(\hat{y}_i) = \frac{exp(\hat{y}_i)}{\sum_{j=1}^cexp(\hat{y}_j)}$. 
	For regression tasks, we choose mean squared error (MSE) as the loss function:
	$\mathcal{L}(y,\hat{y}) =\frac{1}{2}(y-\hat{y})^2$.
	
	\noindent\textit{Evaluation metrics.}
	Micro-F1 and macro-F1 scores are used vis-a-vis evaluation of model performance on multi-class classification tasks.
	They are defined as harmonic mean of micro-average and macro-average of precision, recall respectively. They are formally defined as:
	$p_{micro} = \frac{\sum_{i=1}^cTP_i}{\sum_{i=1}^c(TP_i+FP_i)}, \; r_{micro} = \frac{\sum_{i=1}^cTP_i}{\sum_{i=1}^c(TP_i+FN_i)}, \; p_{macro} = \frac{1}{c}\sum_{i=1}^c\frac{TP_i}{TP_i+FP_i}, \; r_{macro} = \frac{1}{c}\sum_{i=1}^c\frac{TP_i}{TP_i+FN_i}$.
	where $TP_i$, $FN_i$ and $FP_i$ refer to true positives, false negatives, false positives for class $i$ and $c$ denotes the number of classes. For regression, we apply root mean squared error (RMSE) and mean absolute error (MAE) to measure the difference between prediction and ground truth.
		
	\noindent\textit{Baseline models.}
	In our experiments, we compare our proposed models, trained with Algorithm~\ref{algo:1} against generalized linear models (logistic regression for classification and linear regression for regression), FMs~\cite{rendle2010factorization} and $\mathcal{A}^2$~\cite{blondel2016polynomial} on real-world datasets.
	%
	%
	
	\subsection{Effectiveness Analysis}
	\label{subsec:effe}
	From this point on, experimental results are discussed.
	We start with effectiveness of each model in the prediction tasks for the two datasets mentioned in Section~\ref{subsec:data}.
	While showing performance of models for multiple epochs, we only compare the best performance of each model, which might not happen in the same epoch.
	
	\noindent\textit{WDYR.} To evaluate models with the WDYR dataset, we set hyper-parameters as follows: $\lambda_1=0.001$, $\lambda_2=0.1$, $k=10$, $\alpha=0.1$, $\mu=0.1$ and $\gamma=0.5$.
	Each model is trained for 10 epochs with batch size set to 16.
	As shown in Figure~\ref{fig:macro_WDYR}, for macro-F1 scores, SHFMs and SH$\mathcal{A}^2$ outperform FMs and $\mathcal{A}^2$ by $0.51\%$ and $1.00\%$, respectively.
	Figure~\ref{fig:micro_WDYR} shows that SH$\mathcal{A}^2$ is $0.74\%$ better than $\mathcal{A}^2$ w.r.t. micro-F1 score.
	These improvements are not trivial as $\mathcal{A}^2$ only outperforms the logistic regression by $0.24 \%$ in micro-F1 and $0.39 \%$ in macro-F1.
	Even though FMs can achieve performance comparable to that of SHFMs measured by micro-F1, we can conclude that our proposed models outperform their corresponding baselines for the task of cold-start prediction of WDYR.

	\noindent\textit{E2006-tfidf.}
	Regarding the E2006-tfidf dataset, hyper-parameters are set to: $\lambda_1=0.001$, $\lambda_2=0.001$, $k=10$, $\alpha=0.02$, $\mu=0.1$ and $\gamma=0.5$. 
	The training phase lasts 20 epochs for each model in a mini-batch style with batch size 64.
	As shown in Fig.~\ref{fig:rmse_E2006}, in terms of RMSE, SHFMs and SH$\mathcal{A}^2$ outperform FMs and $\mathcal{A}^2$ with $21.15\%$ and $23.32\%$ less error, respectively.
	Similarly, measured by MAE (see Fig.~\ref{fig:mae_E2006}), SHFMs result in $10.83\%$ less error than FMs and SH$\mathcal{A}^2$ leads to $11.99\%$ less error than $\mathcal{A}^2$.
	In brief, for the E2006-tfidf dataset, the experimental results manifest that our proposed hierarchical models outperform their counterparts with unstructured parameters.
	For the task of predicting stock return volatility, SHFMs again achieve the best generalized error.
		Similar to results in~\cite{blondel2016polynomial}, Figure~\ref{fig:perf_WDYR} shows that fixing $\boldsymbol{\beta}=\boldsymbol{1}$ leads to better predictions on testing sets than fitting them as parameters.
	The explanation for this could be that fitting the low-dimensional vector $\boldsymbol{\beta}$ increases the chance of overfitting.
	\begin{figure}[tbh!]
		\centering
		\begin{subfigure}[b]{.49\columnwidth}
			\centering
			\includegraphics[width=1\columnwidth]{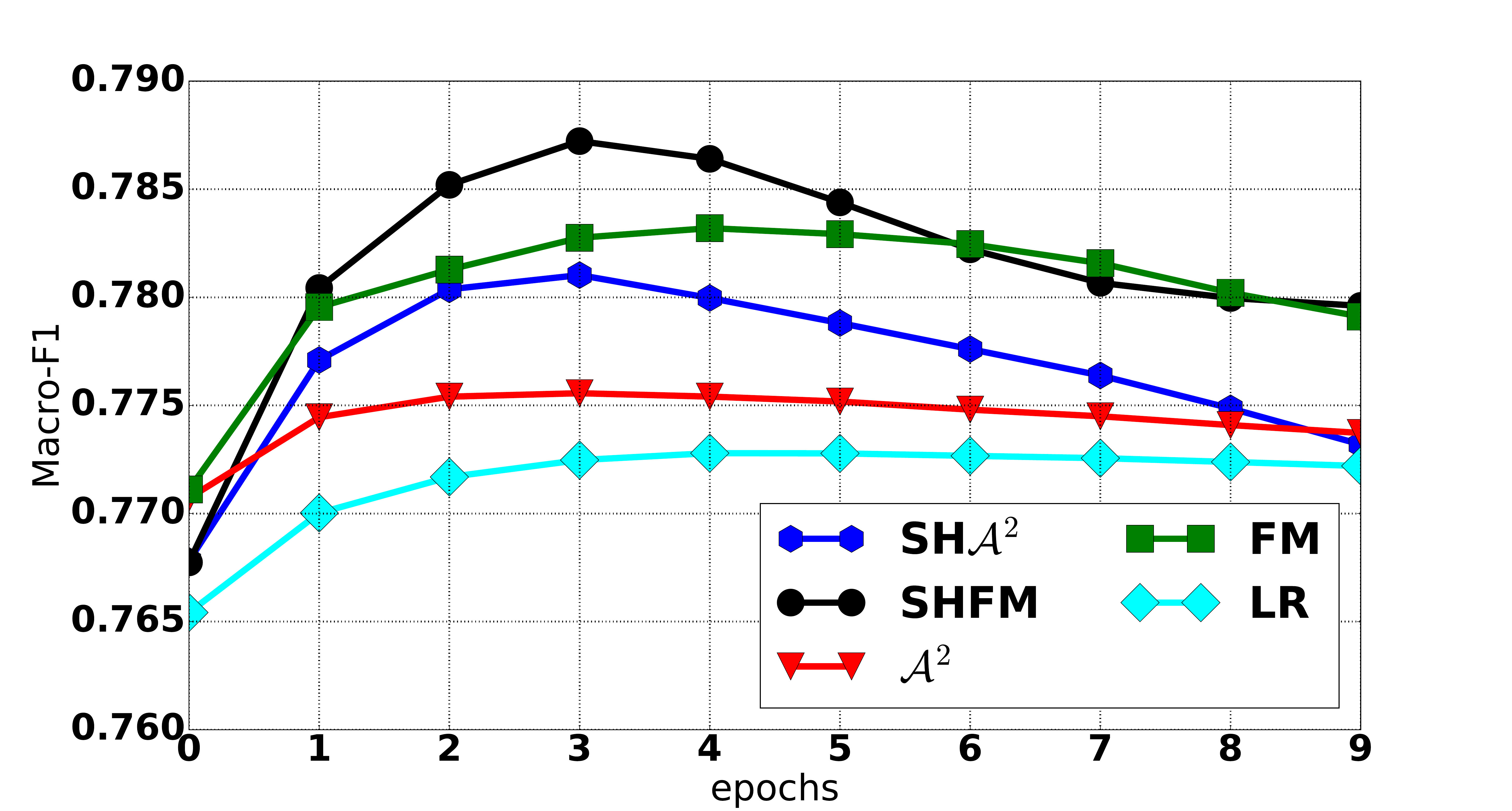}
			\caption{Macro-F1 (WDYR)}
			\label{fig:macro_WDYR}
		\end{subfigure}
	\hfill
		\begin{subfigure}[b]{.49\columnwidth}
			\centering
			\includegraphics[width=1\columnwidth]{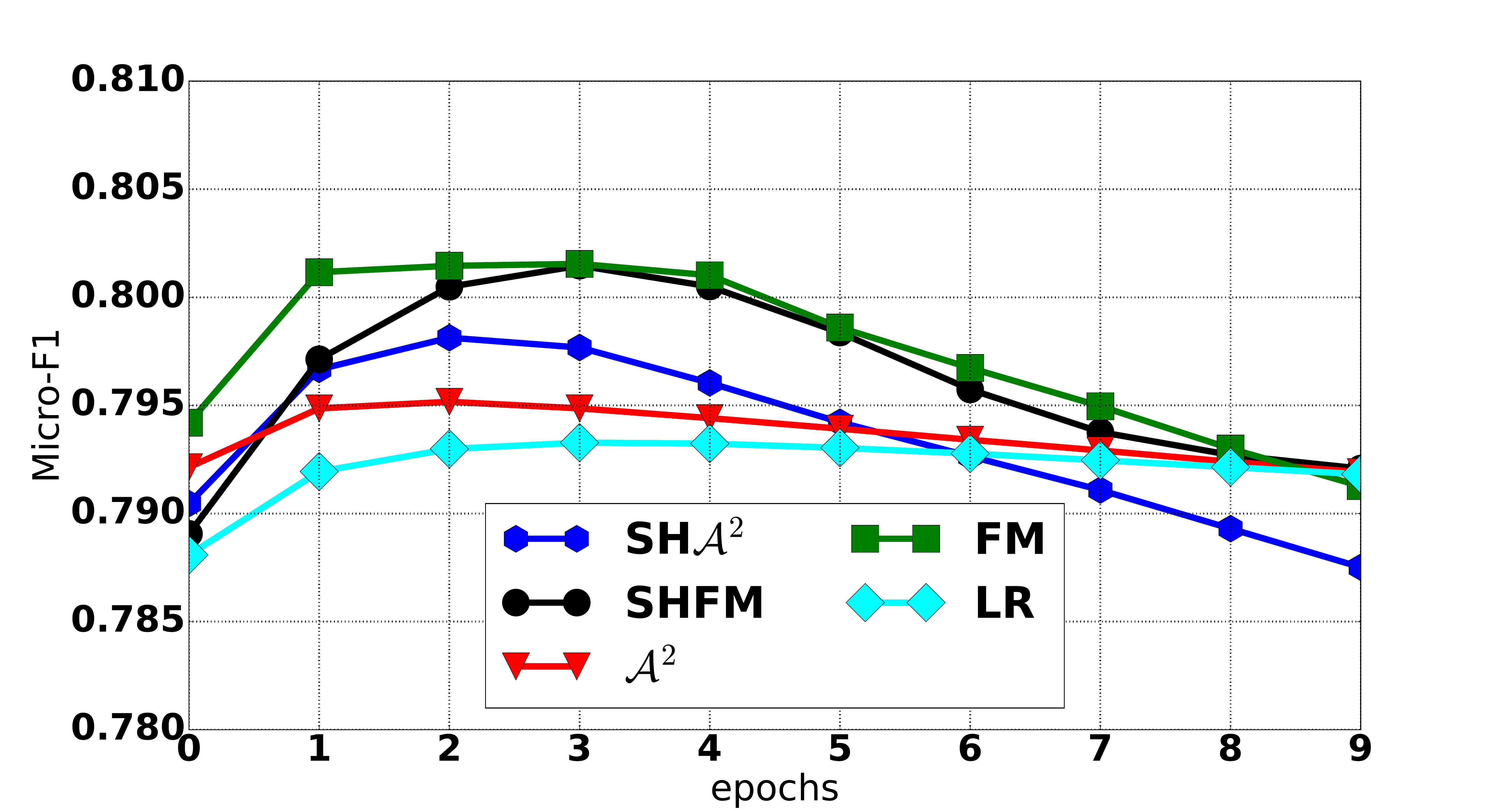}
			\caption{Micro-F1 (WDYR)}
			\label{fig:micro_WDYR}
		\end{subfigure}
		
%
		\begin{subfigure}[b]{.49\columnwidth}
			\centering
			\includegraphics[width=1\columnwidth]{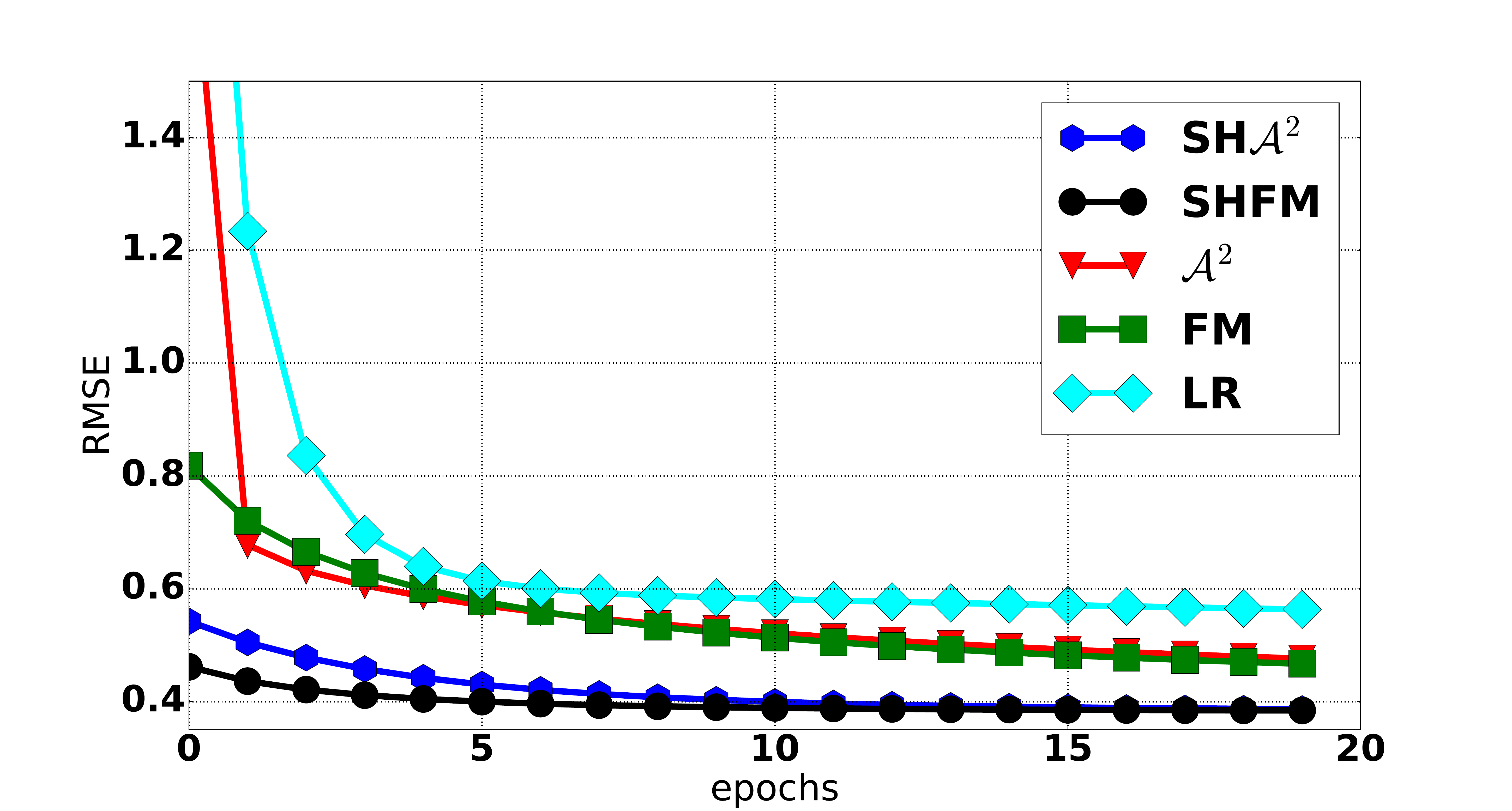}
			\caption{RMSE (E2006)}
			\label{fig:rmse_E2006}
		\end{subfigure}
	\hfill
		\begin{subfigure}[b]{.49\columnwidth}
			\centering
			\includegraphics[width=1\columnwidth]{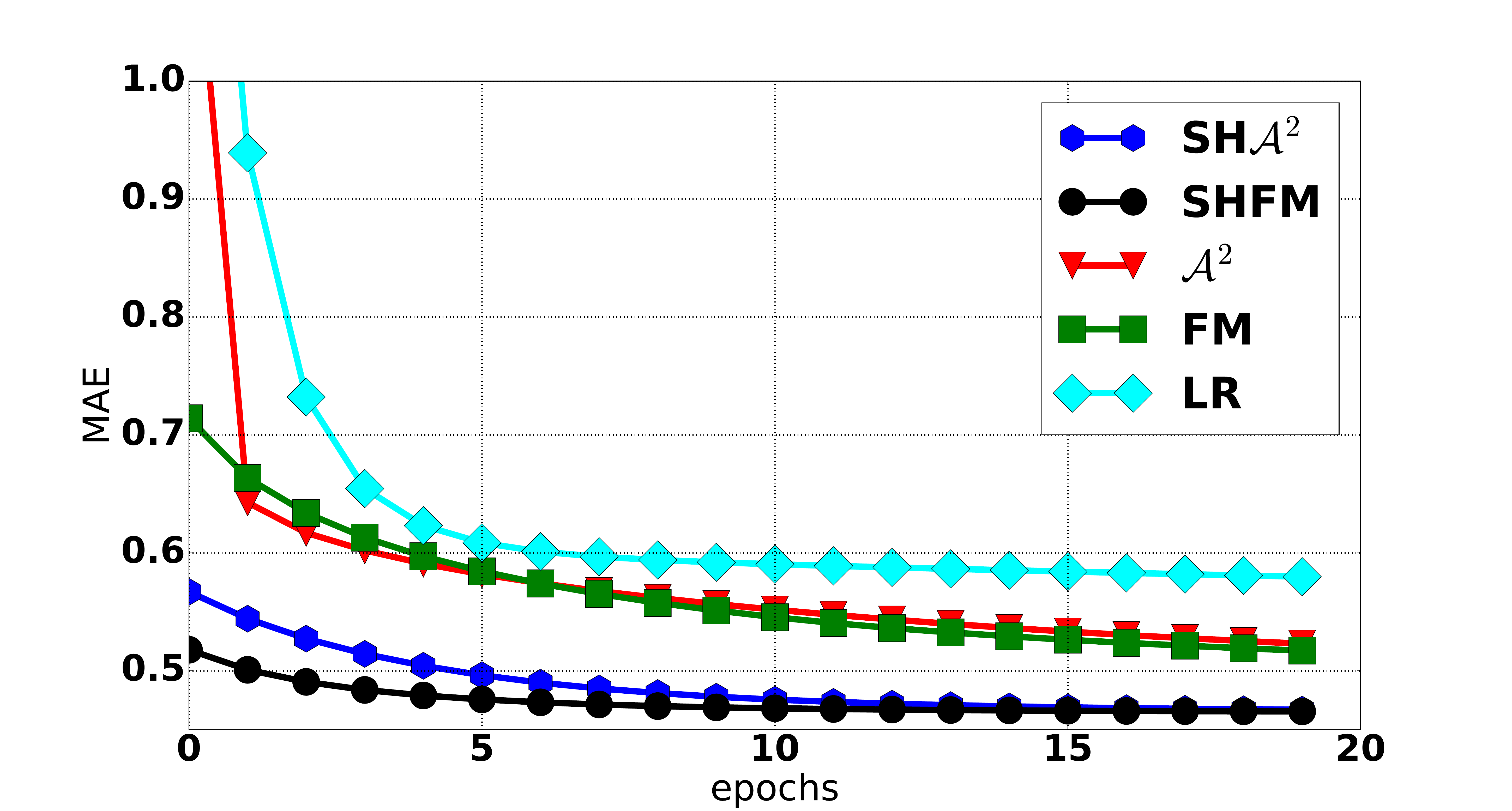}
			\caption{MAE (E2006)}
			\label{fig:mae_E2006}
		\end{subfigure}
\caption{Prediction performance for the two tasks.}
\label{fig:perf_WDYR}
	\end{figure}
	
	
	\subsection{Sensitivity Analysis}
	Experiments are carried out for understanding how hyper-parameters affect the proposed models.
	Experimental results show our models trained with Algorithm~\ref{algo:1} can perform well even with high sparsity in the study of $\lambda_1$.
	In addition, the run time analysis of models with different $k$ provides justification for Proposition~\ref{prop:kd} and~\ref{prop:kd2}.
	To save space, we may only report the experimental results for one of our proposed models on one of the datasets, because similar results are yielded by experiments for the other model or dataset.
	  
	\noindent\textit{The scale of $L_1$ regularization ($\lambda_1$).}
	The scale of $L_1$ regularization ($\lambda_1$) in Algorithm~\ref{algo:1} plays the role of making trade-off between minimizing training loss and maximizing model sparsity. 
	As shown in (\ref{eq:updateV}), given a certain co-ordinate $(i,f)$ and a sample $\boldsymbol{x}^t$, if $|z_{i,f}^t|$ is smaller than $\lambda_1$, the corresponding parameter $V_{i,f}^t$ would become zero after the update.
	Therefore, as the value of $\lambda_1$ increases, the proposed models become sparser.
	To study the influence of $\lambda_1$, we train our proposed models with various values of $\lambda_1$ (see Appendix).
	We set other hyper-parameters as described in~\ref{subsec:effe}.
	By results shown in Table~\ref{tab:l1}, we find that SHFMs and SH$\mathcal{A}^2$ trained with Algorithm~\ref{algo:1} can be both sparse and effective.
	In detail, for the WDYR dataset, SHFMs and SH$\mathcal{A}^2$ achieve the best testing error with sparsity equal to 0.631 and 0.704, respectively.  
	Then again, for the E2006 dataset, even if the models that achieve the best RMSE and MAE on the testing set are dense.
	The very sparse SHFMs (0.999 sparsity) and SH$\mathcal{A}^2$ (0.937 sparsity) do not cause significant increase in RMSE or MAE on testing samples.
	%
	%
	%
	%
	
	\noindent\textit{The number of latent dimensions ($k$).}
	The value of $k$ can affect the proposed models in terms of both effectiveness and efficiency. 
	To study impact of $k$ for the proposed models, we fix other parameters as mentioned in~\ref{subsec:effe} and carry out a series of experiments with different values of $k$. 
	Training loss, testing errors and GPU time per epoch are measured for each value of $k$.
	Each epoch includes training with all training samples and testing on both training and testing sets.
	Intuitively, as the number of latent dimensions increases, the functions that can be represented by both SHFMs and SH$\mathcal{A}^2$ become more abundant.
	Thus, the minimal training loss that our proposed models can reach within a certain number of epochs becomes smaller.
	But this also leads to more expensive computation and may cause the problem of overfitting.
	The experimental results in Fig.~\ref{fig:k_E2006} shows that with larger number of latent dimensions, the proposed models deliver smaller training loss, testing RMSE and MAE.
	However, it is shown that larger value of $k$ can also cause overfitting.
	For the cases of $k=20$ and $k=50$, while the training loss is still dropping, the testing error (measured by RMSE and MAE) starts to increase after several epochs for both proposed models.
	In Table~\ref{tab:k}, we show that the relationship between $k$ and GPU time per epoch is close to linear. 
	This supports Proposition~\ref{prop:kd} and~\ref{prop:kd2}.

	In addition, we monitor whether the assumptions of Proposition~\ref{prop:sh} are held for each value of $\lambda_1$ and $k$ in grid search (see Appendix) and find the probability of cases with either $\boldsymbol{v}_0 = \boldsymbol{0}$ or $\boldsymbol{v}_i \perp \boldsymbol{v}_0, \boldsymbol{v}_i \not = \boldsymbol{0}$ is negligible compared to that of $\boldsymbol{v}_i = \boldsymbol{0}$ which can reach more than 0.999 when $\lambda_1$ is large.
	Therefore, these observations provide justification for our assumptions.
	%

	
	\begin{figure}[tbh!]
		\centering
		\begin{subfigure}[b]{.49\columnwidth}
			\centering
			\includegraphics[width=1\columnwidth]{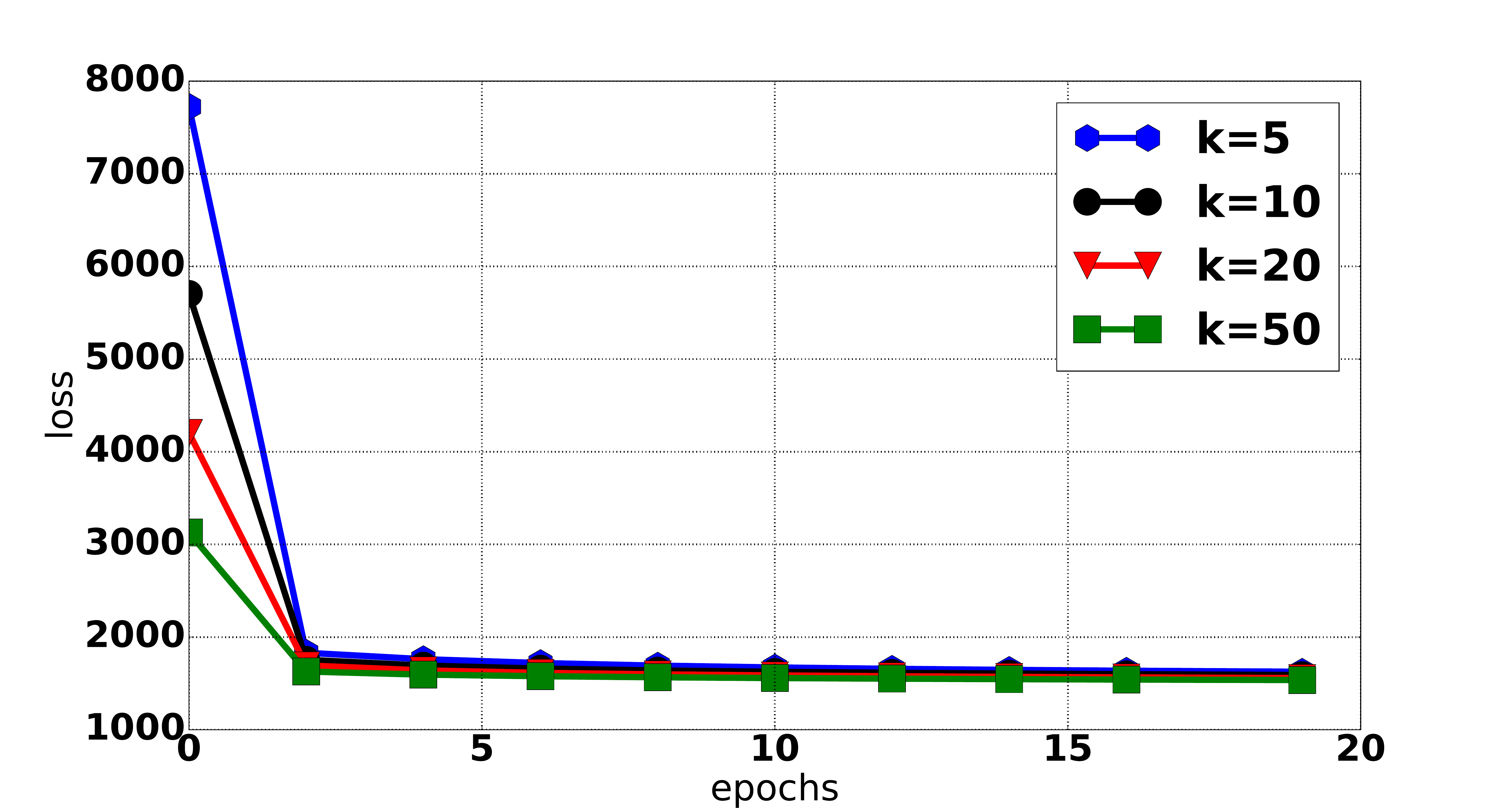}
		\end{subfigure}
	\hfill
		\begin{subfigure}[b]{.49\columnwidth}
			\centering
			\includegraphics[width=1\columnwidth]{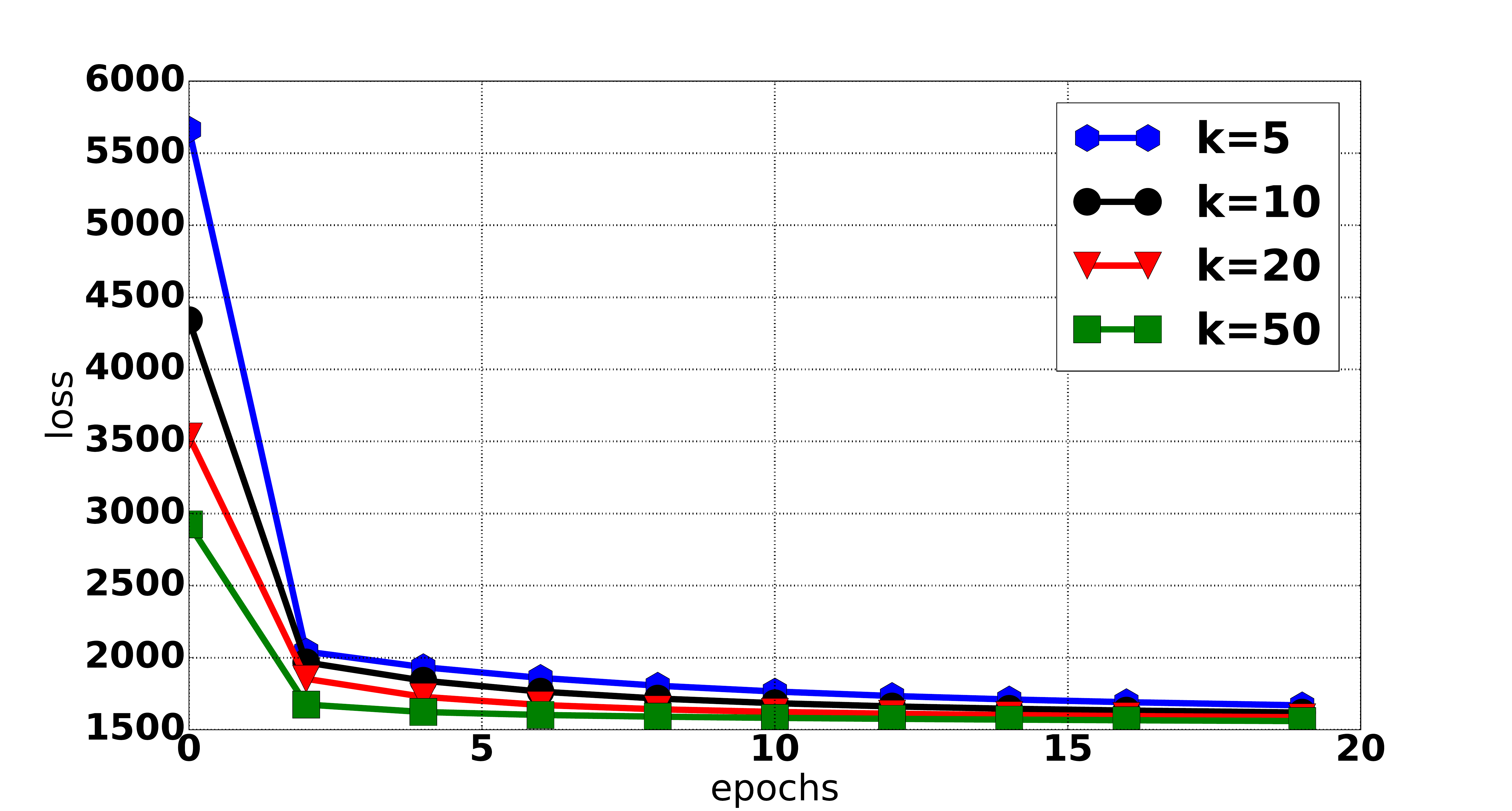}
		\end{subfigure}
		\begin{subfigure}[b]{.49\columnwidth}
			\centering
			\includegraphics[width=1\columnwidth]{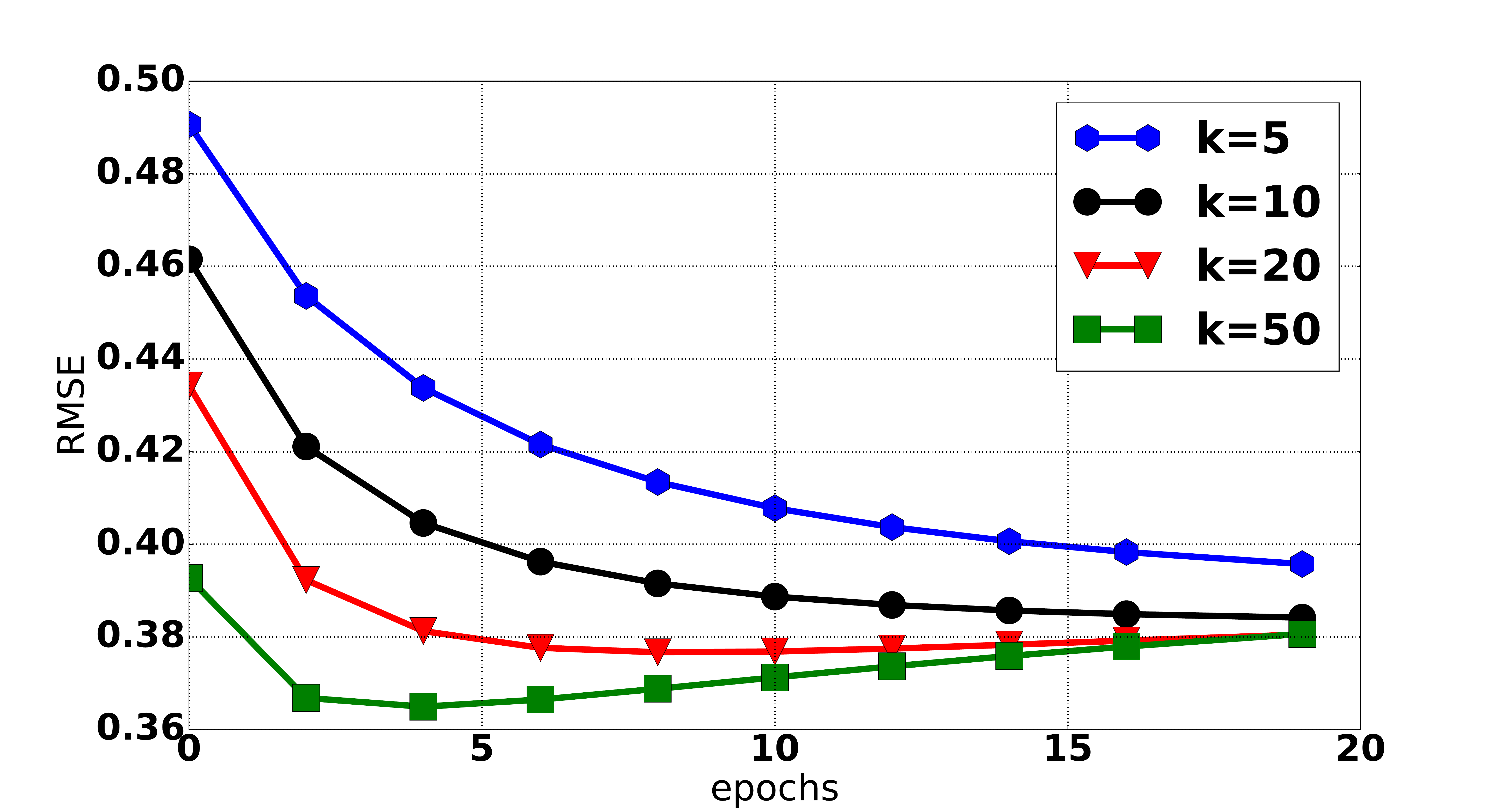}
		\end{subfigure}
	\hfill
	\begin{subfigure}[b]{.49\columnwidth}
		\centering
		\includegraphics[width=1\columnwidth]{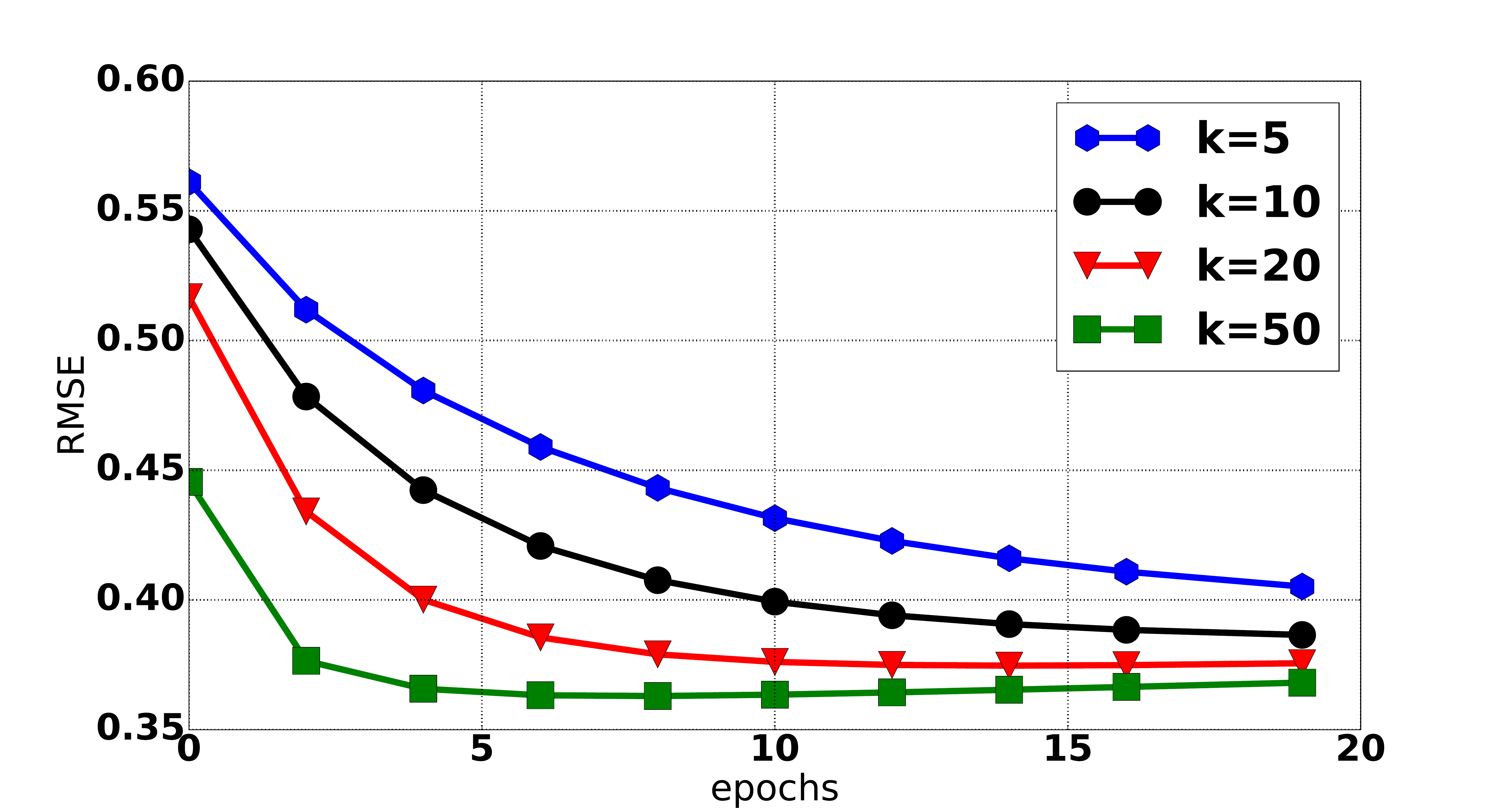}
	\end{subfigure}
	\begin{subfigure}[b]{.49\columnwidth}
	\centering
	\includegraphics[width=1\columnwidth]{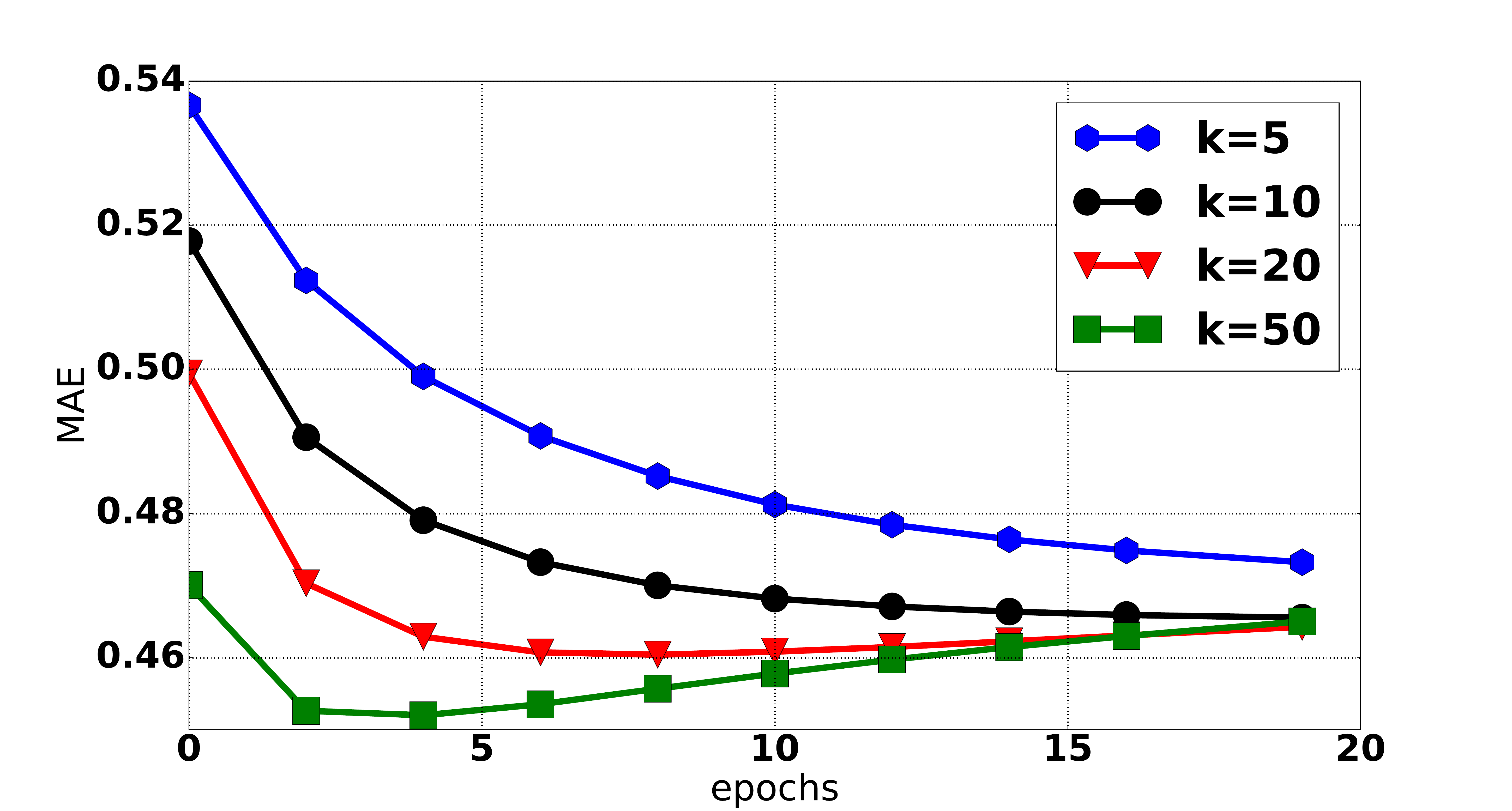}
	\end{subfigure}
	\hfill
		\begin{subfigure}[b]{.49\columnwidth}
			\centering
			\includegraphics[width=1\columnwidth]{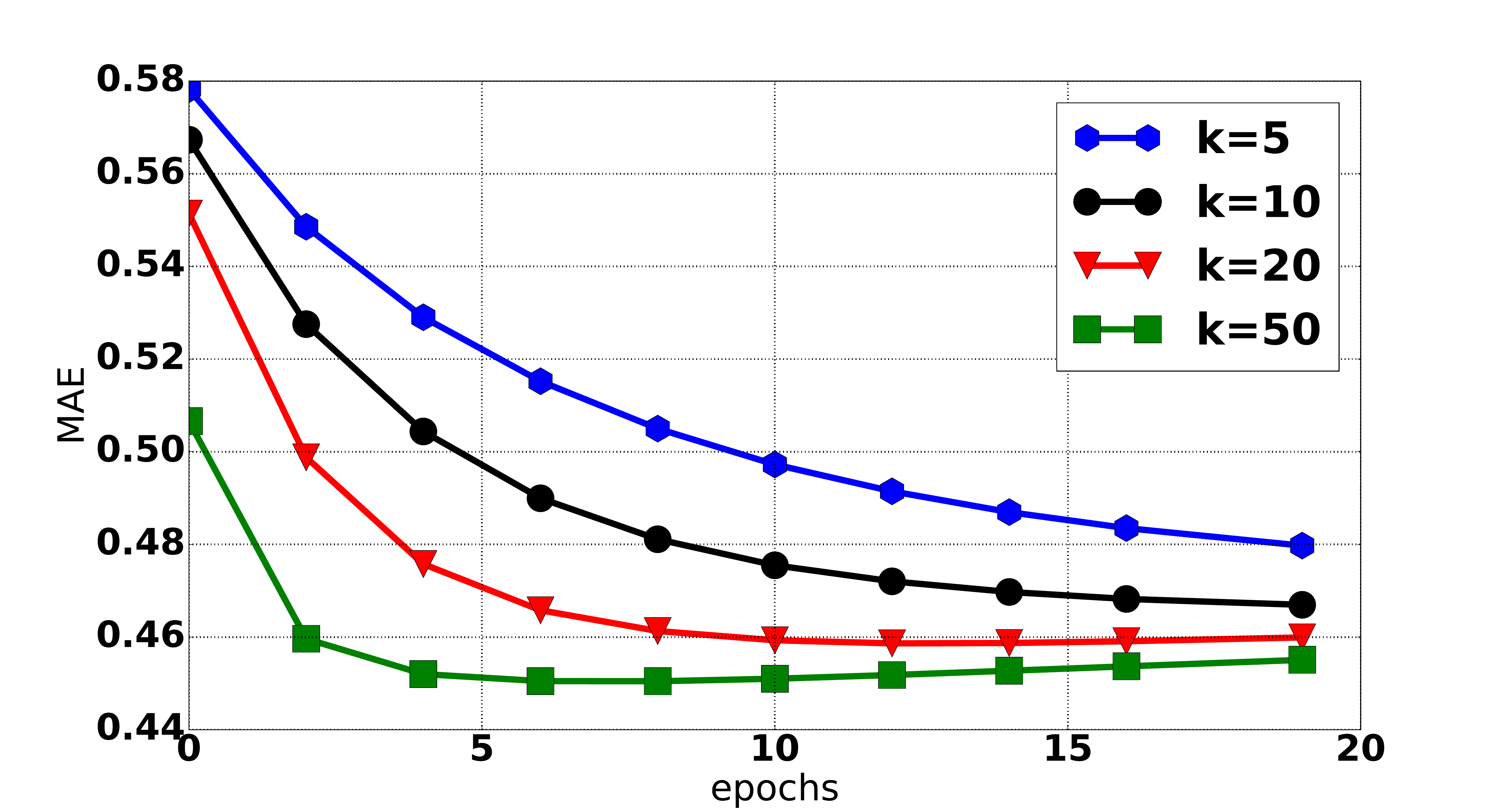}
		\end{subfigure}
		\caption{The impact of the number of latent dimensions ($k$) on training loss, RMSE and MAE of SHFMs and SH$\mathcal{A}^2$ for E2006-tfidf dataset.}
		\label{fig:k_E2006}
	\end{figure}
	
	\begin{table}[tbh!]%
		\small
		\renewcommand{\arraystretch}{1.2}
		\caption{\textmd{Sparsity and Performance with Various $\lambda_1$}}
		\label{tab:l1}
		\centering
		\subcaption{WDYR dataset}
		\begin{tabular}{|c| c| c|c|c|c|}
			\hline
			Model & $\lambda_1$&Sparsity & Loss & Mic-F1 & Mac-F1 \\ \hline
			\multirow{4}{*}{SHFMs}&$10^{-5}$ & 0.630 &0.331 &0.792 &0.776 \\ \cline{2-6} 
			&$10^{-4}$ & 0.631& 0.337& \textbf{0.801}&\textbf{0.787} \\ \cline{2-6} 
			&$10^{-3}$ & 0.854& 0.839&0.758 &0.728 \\ \cline{2-6} 
			&$10^{-2}$ & 0.999& 1.550&0.593 &0.500  \\ \hline
			\multirow{4}{*}{SH$\mathcal{A}^2$}&$10^{-5}$ &0.631 & 0.604 &0.793 &0.777 \\ \cline{2-6} 
			&$10^{-4}$ &0.704 & 0.651& \textbf{0.798}&\textbf{0.781} \\ \cline{2-6} 
			&$10^{-3}$ &0.999 &1.550 & 0.593&0.500 \\ \cline{2-6} 
			&$10^{-2}$ &0.999 &1.550 & 0.593&0.500 \\ \hline
		\end{tabular}
	\smallskip
		\subcaption{E2006 dataset}
		\begin{tabular}{|c| c| c|c|c|c|}
			\hline
			Model & $\lambda_1$&Sparsity & Loss & RMSE & MAE \\ \hline
			\multirow{5}{*}{SHFMs}&$10^{-3}$ & 0.003 &1590.85 &\textbf{0.384} &\textbf{0.466} \\ \cline{2-6} 
			&$10^{-2}$ & 0.018& 1591.93&0.385 &0.466 \\ \cline{2-6} 
			&$0.1$ & 0.137&1591.69 &0.385 &0.466 \\ \cline{2-6} 
			&$1.0$ & 0.582&1596.32 &0.386 &0.467 \\ \cline{2-6}
			&$10.0$ & 0.999& 1632.87&0.386 &0.470 \\ \cline{2-6}
			&$100.0$ & 0.999& 7738.07&2.561 &1.509  \\ \hline
			\multirow{4}{*}{SH$\mathcal{A}^2$}&$10^{-3}$ & 0.002  & 1619.3& \textbf{0.386} &\textbf{0.467} \\ \cline{2-6} 
			&$10^{-2}$ & 0.023 &1619.3 &0.387  &0.467 \\ \cline{2-6} 
			&$0.1$ & 0.192& 1620.12 & 0.387&0.467 \\ \cline{2-6} 
			&$1.0$ & 0.675& 1627.47&0.388 &0.468 \\ \cline{2-6} 
			&$10.0$ &0.937 &1650.46 & 0.390&0.470 \\ \cline{2-6} 
			&$100.0$ &0.995 &  1732.65&0.423 & 0.493 \\ \hline
		\end{tabular}
	\end{table}

\begin{table}[tbh!]%
	\small
	\renewcommand{\arraystretch}{1.2}
	\caption{\textmd{GPU Time per Epoch for different values of $k$}}
	\label{tab:k}
	\centering
	\begin{tabular}{| c|c|c|c|c|}
		\hline
		& \multicolumn{2}{|c|}{SHFMs} & \multicolumn{2}{|c|}{SH$\mathcal{A}^2$}\\ \hline
		 $k$ &WDYR & E2006 & WDYR & E2006 \\ \hline
		$5$ & 49.58 &  5.23&  50.53 & 5.42  \\ 
		$10$ & 76.51 &  7.13 & 67.82  & 7.28\\ 
		$20$ & 142.49&  16.97 &   130.50& 18.09  \\  
		$50$ & 337.44&  38.87 &  342.54&  39.54  \\ \hline
	\end{tabular}
\end{table}
	
\section{Related Work}
\label{sec:rw}
	\noindent\textbf{Factorization Machines and ANOVA kernel regression.}
	In~\cite{rendle2010factorization}, Rendle proposed FMs and showed that it can outperform SVM and PITF~\cite{rendle2010pairwise} in various tasks.
	Variants of FMs have been shown to be effective and efficient in click-through rate prediction~\cite{juan2016field,li2016difacto,pan2016sparse}, recommendation systems~\cite{nguyen2014gaussian,loni2014cross} and microblog retrieval~\cite{qiang2013exploiting}.
	Recently, Blondel et al.~\cite{blondel2016polynomial} showed that FMs belong to the class of ANOVA kernel regression.
	%
	
	\noindent\textbf{Hierarchical structures in parameters.}
	In~\cite{choi2010variable}, Choi et al. extended Lasso~\cite{tibshirani1996regression} with element-wise $L_1$ regularization for imposing strong hierarchy.
	In~\cite{bien2013lasso}, Bien et al. justified weak hierarchy with assumption that a feature is not any more special than its linear transformation. They also proposed to guarantee weak hierarchy with constraints.
	But this method is limited to $L_1$ regularization.
	Similarly, Zhong et al.~\cite{zhong2013efficient} proposed constraints for strong hierarchy and solvers for $L_1$, $L_2$ and $L_{\infty}$ regularization.
	%
	%
	In~\cite{li2016difacto}, Li et al. treated strong hierarchy in FMs as an instance of structured sparsity~\cite{negahban2009unified} and mentioned the method which adds a complicated regularization term.
	However, solving these optimization problems with these constraints or complicated regularization terms are challenging.
	
	\noindent\textit{Online sparse learning.}
	In~\cite{zinkevich2003online}, Zinkevich et al. showed that online gradient descent can be considered as a special case of mirror descent~\cite{beck2003mirror}, which states the closed form update implicitly as an optimization.
	Following this, regarding learning sparse models for large-scale data, algorithms such as COMID~\cite{duchi2011adaptive} and FTRL class algorithms are proposed.
	In~\cite{mcmahan2010unified,mcmahan2011follow}, FTRL-Proximal is claimed to be the most efficient algorithm amongst them in terms of producing sparsity. 
\section{Conclusion}
\label{sec:conclu}
In this work, we propose SHFMs and SH$\mathcal{A}^2$, where strong hierarchy is imposed without extra constraints or complicated regularization terms.  
A FTRL-Proximal algorithm is also derived for learning these models.
Analysis and experiments are done to show that it only takes linear time and space complexity to train the models with the algorithm or make a prediction with the models.
Evaluated with two data mining tasks, we conclude that the proposed models outperform FMs, $\mathcal{A}^2$ and the generalized linear models.
Furthermore, our models can reach high sparsity without significant loss of performance when trained by the algorithm we derive.
For future work, we plan to apply these models to other challenging data mining tasks such as recommendation systems and heterogeneous data mining.






\bibliographystyle{siamplain}
\bibliography{bib}

\end{document}